\documentclass{article}
\usepackage{kr}

\usepackage{times}
\usepackage{soul}
\usepackage{url}
\usepackage[hidelinks]{hyperref}
\usepackage[utf8]{inputenc}
\usepackage[small]{caption}
\usepackage{graphicx}
\usepackage{amsmath}
\usepackage{amsthm}

\usepackage{booktabs}
\usepackage{algorithm}
\usepackage{algorithmic}
\usepackage{latexsym}
\usepackage{array}
\usepackage{amssymb}
\usepackage{color}

\newcommand{\args}{\ensuremath{\mathcal{A}}}
\newcommand{\edges}{\ensuremath{\mathcal{E}}}
\newcommand{\weight}{\ensuremath{\operatorname{w}}}
\newcommand{\probDists}{\ensuremath{\mathcal{P}_\args}}

\newcommand{\attacker}{\ensuremath{\mathrm{Att}}}
\newcommand{\supporter}{\ensuremath{\mathrm{Sup}}}
\newcommand{\constraints}{\ensuremath{\mathcal{C}}}
\newcommand{\metahypotheses}{\ensuremath{\args_M}}
\newcommand{\subhypotheses}{\ensuremath{\args_S}}
\newcommand{\piecesofevidence}{\ensuremath{\args_E}}
\newcommand{\ievidence}{\ensuremath{\mathrm{E_{\mathrm{inc}}}}}
\newcommand{\eevidence}{\ensuremath{\mathrm{E_{\mathrm{ex}}}}}
\newcommand{\innocence}{\ensuremath{\mathrm{Innocence}}}
\newcommand{\devidence}{\ensuremath{\mathrm{E_{\mathrm{d}}}}}
\newcommand{\cevidence}{\ensuremath{\mathrm{E_{\mathrm{c}}}}}
\newcommand{\alibi}{\ensuremath{\mathrm{Alibi}}}
\newcommand{\ability}{\ensuremath{\mathrm{Ability}}}
\newcommand{\motive}{\ensuremath{\mathrm{Motive}}}
\newcommand{\opportunity}{\ensuremath{\mathrm{Opportunity}}}
\newcommand{\cameraf}{\ensuremath{\mathrm{Camera1}}}
\newcommand{\cameras}{\ensuremath{\mathrm{Camera2}}}
\newcommand{\camera}{\ensuremath{\mathrm{Camera}}}

\newcommand{\metahypothesesedges}{\ensuremath{\edges_M}}
\newcommand{\supportedges}{\ensuremath{\edges_S}}
\newcommand{\evidenceedges}{\ensuremath{\edges_E}}
\newcommand{\belief}{\ensuremath{\mathcal{B}}}
\newcommand{\ubelief}{\ensuremath{\overline{\belief}}}
\newcommand{\lbelief}{\ensuremath{\underline{\belief}}}

\newtheorem{proposition}{Proposition}

\theoremstyle{definition}
\newtheorem{definition}{Definition}
\newtheorem{example}{Example}

\author{%
	Inga Ibs$^1$\and
	Nico Potyka$^2$\\
	\affiliations
	$^1$Technical University of Darmstadt\\
	$^2$University of Stuttgart\\
	\emails
	inga.ibs@tu-darmstadt.de,
	nico.potyka@ipvs.uni-stuttgart.de
}
\title{Explainable Automated Reasoning in Law using Probabilistic Epistemic Argumentation }

\begin{document}
\maketitle

\begin{abstract}
 Applying automated reasoning tools for decision support and analysis in law 
 has the potential to make court decisions more transparent and objective. 
 Since there is
 often uncertainty about the accuracy and relevance of evidence,
 non-classical reasoning approaches are required. 
 Here, we investigate probabilistic epistemic argumentation as a tool for automated reasoning about legal cases. We introduce a general scheme to model legal cases as probabilistic epistemic argumentation problems,
explain how evidence can be
 modeled and sketch how explanations for legal decisions can be
 generated automatically. 
 Our framework is easily interpretable, can deal with cyclic
 structures and imprecise probabilities and guarantees polynomial-time
 probabilistic reasoning in the worst-case. 
\end{abstract}

\section{Introduction}

Legal reasoning problems can be addressed from different
perspectives.
From a lawyer's perspective, a trial may be best modeled as a
strategic game.
In a criminal trial, for example, the prosecutor may try
to convince the judge or jury of the defendant's guilt
while the defense attorney tries the opposite.
The problem is then to interpret the law and the evidence
in a way that maximizes the agent's utility. 
From this perspective, a legal reasoning problem is best modeled
using tools from decision and game theory \cite{hanson2014game,prakken1996dialectical,riveret2007success}. 

Our focus here is not on strategic considerations, 
but on the decision process that leads
to the final verdict in a legal process like a trial.
Given different pieces of evidence and beliefs about
their authenticity and relevance, how can we merge them
to make a
plausible and transparent decision?
Different automated reasoning tools have been applied in
order to answer similar questions, for example,
case-based reasoning \cite{bench2003model,mccarty1995implementation},
argumentation frameworks \cite{dung2010towards,prakken2013formalization}
or Bayesian networks \cite{fenton2013general}.
Since lawyers and  judges often struggle with the interpretation
of Bayesian networks, recent work also tries to explain Bayesian networks by argumentation tools \cite{vlek2016method}.

Here, we investigate the applicability of the probabilistic epistemic argumentation framework developed in \cite{hunter2013probabilistic,HunterPT2018Arxiv,HunterT16,thimm2012probabilistic}.
As opposed to classical argumentation approaches, this framework allows
expressing uncertainty by means of probability theory.
In particular, we can compute reasoning results in polynomial time
when we restrict the language \cite{potyka2019fragment}.
As it turns out, the resulting fragment is sufficiently expressive for our purpose,
so that our framework is computationally more efficient than
many other probabilistic reasoning approaches that suffer from
exponential runtime in the worst-case. At the same time, the graphical structure
is easily interpretable and allows to automatically generate
explanations for the final degrees of belief (probabilities) as we will explain later. 

While we can incorporate objective probabilities in our framework,
our probabilistic reasoning is best described as subjective in the sense that we basically merge beliefs about pieces of evidence
and hypotheses (probabilities that can be either objective or subjective). In order to define the beliefs about pieces of 
evidence from objective evidence and statistical information,
another approach like Bayesian networks or more general tools
from probability theory may be better suited. Our framework
can then be applied on top of these tools.
In this sense, our framework can be seen as a complement rather 
than a replacement of alternative approaches.

The remainder of this paper is structured as follows:
Section 2 explains the necessary basics. We will introduce
a basic legal argumentation framework in Section 3 and
discuss more sophisticated building blocks in Section 4.
We will discuss and illustrate the explainability capabilities of our approach 
as we proceed, but explain some more general ideas in Section 5.
Finally, we add some discussion about related work, the
pros and cons of our framework and future work in Sections 6 
and 7.

\section{Probabilistic Epistemic Argumentation Basics}

Our legal reasoning approach builds up on the probabilistic epistemic argumentation approach
developed in \cite{thimm2012probabilistic,hunter2013probabilistic,HunterT16,HunterPT2018Arxiv}.
In this approach, we assign degrees of belief in the form of probabilities to arguments using probability functions
over possible worlds. A possible world basically interprets every argument as either accepted
or rejected. In order to restrict to probability functions that respect prior beliefs and
the structure of the argumentation graph, different constraints can be defined. 
Afterwards, we can assign a probability interval to every argument based on these constraints.
We will restrict to a fragment of the constraint language here that allows polynomial-time
computations \cite{potyka2019fragment}.

Formally, we represent arguments and their relationships in a directed edge-weighted graph 
$(\args, \edges, \weight)$. $\args$ is a finite set of arguments,
$\edges \subseteq \args \times \args$ is a finite set of directed edges between the arguments and
$\weight: \edges \rightarrow \mathbb{Q}$ assigns a rational number to 
every edge. 
If there is an edge $(A,B) \in \edges$, we say that
\emph{$A$ attacks $B$} if $w((A,B)) < 0$ and \emph{$A$ supports $B$} if $w((A,B)) > 0$.
We let 
$\attacker(A) = \{B \in \args \mid (B,A) \in \edges, w((A,B)) < 0\}$ be the set of attackers of an argument A 
and
$\supporter(A) = \{B \in \args \mid (B,A) \in \edges, w((A,B)) > 0\}$ be the set of supporters.

A \emph{possible world} is a subset of arguments $\omega \subseteq \args$. Intuitively,
$\omega$ contains the arguments that are accepted in a particular state of the world. 
Beliefs about the true state of the world are modeled by rational-valued probability
functions $P: 2^\args \rightarrow [0,1]\cap \mathbb{Q}$ such that $\sum_{\omega \in 2^\args} P(\omega) = 1$.
The restriction to probabilities from the rational numbers is for computational reasons only. In practice, it does 
not really mean any loss of generality because
implementations usually use finite precision arithmetic.
We denote the set of all probability functions over $\args$ by $\probDists$.
The probability of an argument $A \in \args$ under $P$
is defined by adding the probabilities of all worlds in which $A$ is accepted, that is, 
$P(A) = \sum_{\omega \in 2^\args, A \in \omega} P(\omega)$. $P(A)$ can be understood as a degree of belief, 
where $P(A) = 1$ means complete acceptance and $P(A)=0$ means complete rejection.

The meaning of attack and support relationships can be defined by means of 
constraints in probabilistic epistemic argumentation.
For example, the \emph{Coherence} postulate in \cite{HunterT16} intuitively demands that 
the belief in an argument is bounded from above by the belief of its attackers.
Formally, a probability function $P$ respects \emph{Coherence}
iff $P(A) \leq 1 - P(B)$ for all $B \in \attacker(A)$.
A more general constraint language has recently been introduced in \cite{HunterPT2018Arxiv}.
Here, we will restrict to a fragment of this language that allows solving our reasoning
problems in polynomial time \cite{potyka2019fragment}. 
A \emph{linear atomic constraint} is an expression of the form 
$$c_0 + \sum_{i=1}^n c_i \cdot \pi(A_i) \leq d_0 + \sum_{i=1}^m d_i \cdot \pi(B_i),$$
where $A_i, B_i \in \args$, $c_i, d_i \in \mathbb{Q}$, $n,m \geq 0$ (the sums can be empty)
and $\pi$ is a syntactic symbol that can be read as
'the probability of'.
For example, the \emph{Coherence} condition above can be expressed by a linear atomic constraint
with $m=n=1$, $c_0=0$, $c_1=1$, $A_1 = A$, $d_0 = 1$, $d_1 = -1$ and $B_1=B$. However, we can also define more complex
constraints that take the beliefs of more than just two arguments into account.
Usually, the arguments that occur in a constraint are neighbors in the graph
and the coefficients $c_i, d_i$ will often be based on the weight of the edges between the arguments.
We will see many examples later.

A probability function $P$ \emph{satisfies} a linear atomic constraint iff
$c_0 + \sum_{i=1}^n c_i \cdot P(A_i) \leq d_0 + \sum_{i=1}^m d_i \cdot P(B_i)$.
$P$ satisfies a set of linear atomic constraints $\constraints$, denoted as $P \models \constraints$, iff it satisfies all constraints $c \in C$.
If this is the case, we call $\constraints$ \emph{satisfiable}.

We are interested in two reasoning problems here that have been introduced in 
\cite{HunterT16}.
First, the \emph{satisfiability problem} is, given a graph $(\args, \edges, \weight)$
and a set of constraints $\constraints$ over this graph, to decide if the constraints
are satisfiable. This basically allows us to check that our modelling assumptions are 
consistent.
Second,  the \emph{entailment problem} is, given a graph $(\args, \edges, \weight)$,
a set of satisfiable constraints $\constraints$ and an argument $A$, to compute lower and upper bounds
on the probability of $A$ based on the probability functions that satisfy the constraints.
For example, suppose we have $\args = \{A, B, C\}$, $\edges = \{(A,B), (B,C)\}$, $\weight((A,B)) = 1$, 
$\weight((B,C)) = -1$. We encode the meaning of the support relationship $(A,B)$ by
$w((A,B)) \cdot \pi(A) \leq \pi(B)$ (a supporter bounds the belief in the argument from below) and 
the meaning of the attack relationship $(B,C)$ by $\pi(C) \leq 1 + w((B,C)) \cdot P(B)$ (an attacker bounds the belief in the argument from above). Say, we also tend to accept $C$ and model this
by the constraint $0.5 \leq \pi(C)$. 
Then our constraints are satisfiable and the entailment
results are $P(A) \in [0, 0.5]$, $P(B) \in [0, 0.5]$, $P(C) \in [0.5, 1]$.
To understand the reasoning, let us consider the upper bound for $A$.
If we had $P(A) > 0.5$, we would also have $P(B) > 0.5$ because of the support 
constraint. But then, we would have $P(C) < 0.5$ because of the attack constraint. 
However, this would violate our constraint for $C$. Hence, we must have $P(A) \leq 0.5$.
In particular, if we would add the constraint $1 \leq \pi(A)$ (accept $A$), our
constraints would become unsatisfiable.
Both the satisfiability and the entailment problem can be automatically solved by 
linear programming techniques.
In general, the linear programs can become exponentially large.
However, both problems
can be solved in polynomial time when we restrict to linear atomic constraints \cite{potyka2019fragment}.

\section{Basic Legal Argumentation Framework}

Legal reasoning problems can occur in many forms and an attempt to capture all of them
at once would most probably result in a framework that is hardly more
concrete than a general abstract argumentation framework.
We will therefore focus on a particular scenario, where the innocence of a defendant
has to be decided.
Modeling a single case may not be sufficient to illustrate the general applicability
of probabilistic epistemic argumentation.
We will therefore try to define a reasoning framework that can be instantiated for different cases, while still being easily  comprehensible.
As with every formal model, there are some simplifying assumptions
about the nature of a trial. 
However, we think that our framework
is sufficient to illustrate how real cases can be modeled and
structured by means of probabilistic epistemic argumentation.
We will make some additional comments about this as we proceed.

Following \cite{fenton2013general}, we regard a legal case roughly as
a collection of \emph{hypotheses} and \emph{pieces of evidence}
that support the hypotheses. 
We model both as abstract arguments,
that is, as something that can be accepted or rejected to a certain degree
by a legal
decision maker like a judge, the jury or a lawyer. 
To begin with, we introduce
three meta hypotheses that we model by 
three arguments 
$\ievidence$ (the defendant should be declared guilty because of the inculpatory evidence),
$\eevidence$ (the defendant should be declared innocent because of the exculpatory evidence) and
$\innocence$ (the defendant is innocent).
We regard $\innocence$ as the ultimate hypothesis
that is to be decided within the trial.
In general, it may be necessary to consider several
ultimate hypotheses that may correspond to different qualitative
degrees of legal liability (e.g. intent vs. accident vs. innocent).
If necessary, these can be incorporated by adding additional ultimate
hypotheses in an analogous way.
$\ievidence$ and $\eevidence$ are supposed to merge 
hypotheses and pieces of evidence that speak
against ($\ievidence$) or for ($\eevidence$)
the defendant's innocence
as illustrated in Figure \ref{fig:metagraph}.
\begin{figure}[t]
	\centering
		\includegraphics[width=0.45\textwidth]{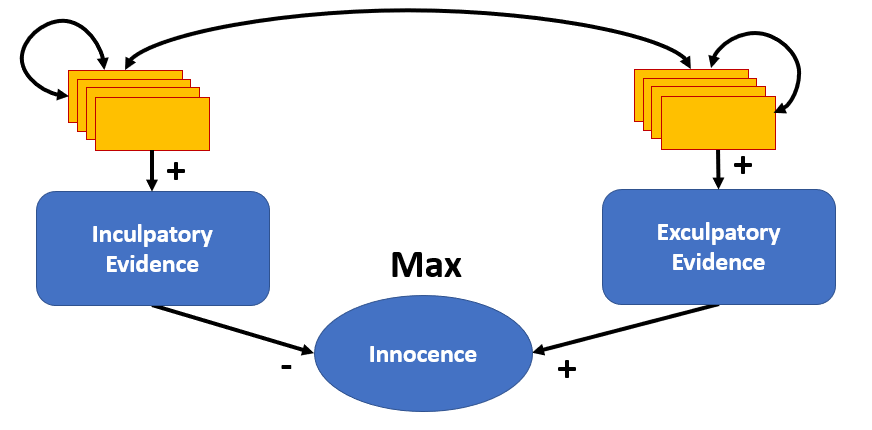}
	\caption{Meta-Graph for our Legal Reasoning Framework.}
	\label{fig:metagraph}
\end{figure}
Support relationships are indicated by a plus and attack relationships by a minus sign.
There can also be attack and support relationships between
pieces of evidence and additional hypotheses.

Intuitively, as our belief in $\ievidence$ increases, our belief in 
$\innocence$ should decrease.
As our belief in $\eevidence$ increases,
our belief in $\innocence$ should increase.
From a classical perspective, accepting $\ievidence$, 
should result in rejecting $\innocence$ and 
accepting $\eevidence$, 
should result in accepting $\innocence$.
In particular, we should not accept 
$\eevidence$ and $\ievidence$ at the same time.
Of course, in general, both the inculpatory 
evidence and the exculpatory evidence can be convincing to a certain degree. 
Probabilities are one natural way to capture this uncertainty.
Intuitively, our basic framework is based on the 
following assumptions that
we will make precise in the subsequent definition.
\begin{description}
    \item[Inculpatory Evidence (IE):] The belief in $\innocence$ is bounded
    from above by the belief in $\ievidence$.
    \item[Exculpatory Evidence (EE):] The belief in $\innocence$ is bounded
    from below by the belief in $\eevidence$.
    \item[Supporting Evidence (SE):] The belief in $\ievidence$ and $\eevidence$ is bounded from
    below by the belief in their supporting pieces of
    evidence.
    \item[Presumption of Innocence (PI):] The belief in $\innocence$ 
    is the maximum belief that is consistent with all assumptions.
\end{description}
The following definition gives a more formal description
of our framework. Our four main assumptions
are formalized in items 4 and 5. 
\begin{definition}[Basic Legal Argumentation Framework (BLAF)]
\label{def_blaf}
A BLAF is a quadruple $(\args, \edges, \weight, \constraints)$,
where $\args$ is a finite set of arguments,
$\edges$ is a finite set of directed edges between the arguments,
$\weight: \edges \rightarrow \mathbb{Q}$ is a weighting function 
and $\constraints$ is a set of linear atomic constraints over $\args$
such that:
\begin{enumerate}
    \item $\args = \metahypotheses \uplus \subhypotheses \uplus \piecesofevidence$
     is partitioned into a set of \emph{meta-hypotheses}
    $\metahypotheses = \{\innocence, \ievidence, \eevidence\},$
     a set of \emph{sub-hypotheses} $\subhypotheses$ 
     and a set of \emph{pieces of evidence} $\piecesofevidence$. 
    \item $\edges = \metahypothesesedges \uplus \supportedges \uplus \evidenceedges$ 
    is partitioned into a set of \emph{meta edges}
    $\metahypothesesedges = \{
    (\ievidence, \innocence),
    (\eevidence, \innocence)\}$, a set of \emph{support edges} 
    $\supportedges \subseteq (\subhypotheses \cup \piecesofevidence) \times \{\ievidence, \eevidence\}$
    and a set of \emph{evidential edges} 
    $\evidenceedges \subseteq (\subhypotheses \cup \piecesofevidence) \times (\subhypotheses \cup \piecesofevidence)$.
    \item $\weight((\ievidence, \innocence))=-1$ and
    $\weight((\eevidence, \innocence))=1$.
     Furthermore, $0 \leq \weight(e) \leq 1$ for all $e \in \supportedges$
    \item $\constraints$ contains at least the following constraints:
    \begin{description} 
        \item[IE:] $\pi(\innocence) \leq 1 + \weight((\ievidence, \innocence)) \cdot \pi(\ievidence)$,
        \item[EE:] $ \weight((\eevidence, \innocence)) \cdot \pi(\eevidence) \leq \pi(\innocence)$,
        \item[SE:] $\weight((E,H)) \cdot \pi(E) \leq \pi(H)$ for all $(E,H) \in \supportedges$.
    \end{description}
 \item For all $A \in \args$, we call $\lbelief(A) = \min_{P \models \constraints} P(A)$
 the \emph{lower belief in $A$} and $\ubelief(A) = \max_{P \models \constraints} P(A)$
 the \emph{upper belief in $A$}. The \emph{belief in $\innocence$} in is defined as
 \begin{equation*}
     PI: \belief(\innocence) = \ubelief(A). 
 \end{equation*}
 and the \emph{belief} in the remaining $A \in \args \setminus \{\innocence\}$ is the interval
 $\belief(A) = [\lbelief(A), \ubelief(A)]$. 
\end{enumerate}
\end{definition}
Items 1-3 basically give a more precise description of the graph illustrated in 
Figure \ref{fig:metagraph}. Item 4 encodes our first three main assumptions as linear
atomic constraints. 
The general form of our basic constraints is 
 $\pi(B) \leq 1 + w((A,B)) \cdot P(A)$ for attack relations $(A,B)$
 (note that for $w((A,B)) = -1$, this is just the coherence constraint from \cite{HunterT16})
and $w((A,B)) \cdot \pi(A) \leq \pi(B)$ for support relations.
Intuitively, attacker bound beliefs from above and supporter bound beliefs
from below.
Item 5 defines lower and upper beliefs in arguments as the minimal
and maximal probabilities that are consistent with our constraints.
Following our fourth assumption (presumption of innocence), the belief in $\innocence$ is defined 
by the upper bound. The beliefs in the remaining arguments is the interval defined by the lower
and upper bound.
The following proposition summarizes some
consequences of our basic assumptions.
\begin{proposition}
\label{prop_basic_beliefs}
For every BLAF $(\args, \edges, \weight, \constraints)$,  we have
\begin{enumerate}
    \item $\ubelief(\ievidence) \leq 1 - \lbelief(\eevidence)$ and $\ubelief(\eevidence) \leq 1 - \lbelief(\ievidence)$.
    \item For all support edges $(a, E) \in \supportedges$, we have 
    \begin{itemize}
        \item $\ubelief(\eevidence) \leq 1 - \weight((a,\ievidence)) \cdot \lbelief(a)$ if $E = \ievidence$,
        \item $\ubelief(\ievidence) \leq 1 - \weight((a,\eevidence)) \cdot \lbelief(a)$ if $E = \eevidence$.
    \end{itemize}
\end{enumerate}
\end{proposition}
\begin{proof}
1. We prove only the first statement, the second one follows analogously. 
 Consider an arbitrary $P \in \probDists$ that satisfies $\constraints$. 
  Then $P(\ievidence) \leq P(\innocence) 
    \leq 1 - P(\eevidence) \leq 1 - \lbelief(\eevidence)$. The first inequality
    follows from EE and the second from IE (Def. \ref{def_blaf}, item 4) along with the 
    conditions on $\weight$ (Def. \ref{def_blaf}, item 3). The third inequality follows
    because $\lbelief(\eevidence) \leq P(\eevidence)$ by definition of $\lbelief$.

2. Again, we prove only the first statement.
  Note that SE (Def. \ref{def_blaf}, item 4) implies
  $P(\ievidence) \geq \weight((a,\ievidence)) \cdot P(a)$
  for all $P \in \probDists$ that satisfy $\constraints$.
  Therefore, $P(\eevidence) \leq 1 - P(\ievidence)
   \leq 1 - \weight((a,\ievidence)) \cdot P(a) \leq 1 - \weight((a,\ievidence)) \cdot \lbelief(a) $, where the first and third inequalities can be derived like in 1.
\end{proof}
Intuitively, item 1 says that our upper belief that the defendant should be declared guilty
because of the inculpatory evidence is bounded from above by our lower belief that the defendant
should be declared innocent because of the exculpatory evidence and vice versa.
By rearranging the equations, we can see that the lower belief in $\ievidence$
is also bounded from above by the upper belief in $\eevidence$ and vice versa.
Item 2 explains that every argument $a$ that directly contributes to inculpatory (exculpatory) evidence $E$ 
gives an upper bound for the belief in $\eevidence$ ($\ievidence$) that is based on our lower 
belief $\lbelief(a)$ and the relevance $\weight((a,E))$ of this argument.
In a similar way, we could bound the beliefs in contributors to $\ievidence$ by the belief in
contributors to $\eevidence$ by taking their respective weights into account. 
However, the general description becomes more and more difficult
to comprehend. Therefore, we just illustrate the interactions by means of a simple example.
\begin{example}
\label{example_blaf}
Let us consider a simple case of hit-and-run driving. The defendant is accused of having struck a car while parking at a shopping center. The plaintiff witnessed the accident from afar and denoted the registration number from the licence plate when the car left ($T_1$). The defendant denies the crime and testified that he was at home with his girlfriend at the time of the offence ($T_2$). His girlfriend confirmed his alibi ($T_3$). However, a security camera at the parking space recorded a person that bears strong resemblance to the defendant at the time of the crime ($E_1$). We consider a simple formalization shown in Figure \ref{fig:example_blaf}.
\begin{figure}[tb]
	\centering
		\includegraphics[width=0.42\textwidth]{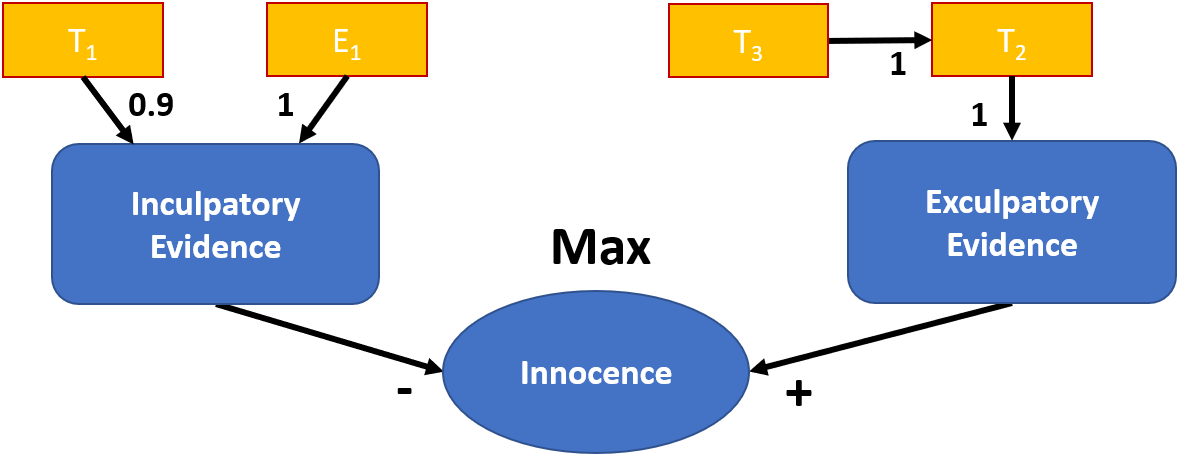}
	\caption{BLAF for Example \ref{example_blaf}.}
	\label{fig:example_blaf}
\end{figure}
We designed the graph in a way that allows illustrating the interactions in our framework. 
One may also want to regard $T_3$ as a supporter of exculpatory evidence and consider attack relationships between $E_1$ and
$T_1$ and $T_3$. We do not introduce such edges because we want to illustrate the indirect interactions
between arguments. In this example, we may weigh all edges with $1$ and control the uncertainty only about
the degrees of belief. However, we assign a weight of $0.9$ to the edge from $T_1$ in order to illustrate the effect
of the weight. This may capture the uncertainty that the plaintiff may have written down the wrong
registration number, for example.
The probability for $T_1$, $T_2$ and $T_3$ is our degree of belief that the corresponding testimonies are true. The probability of $E_1$ is our degree of belief that the camera does indeed show the 
defendant and not just another person. 
Without additional assumptions, we can only derive that our degree of belief in $\innocence$ is
$1$ (presumption of innocence) as shown in the second column ($\belief_1$) of Table \ref{tab:example_blaf}.
\def\arraystretch{1.2}
\begin{table}
	\centering
	\begin{tabular}{l>{\raggedleft}p{0.9cm}>{\raggedleft}p{1.2cm}>{\raggedleft \arraybackslash}p{1.2cm}}
		$\args$ & $\belief_1$ & $\belief_2$ & $\belief_3$\\ 
		\hline
		$\innocence$ & 1     & 1     & 0.1  \\
		$\ievidence$ & [0, 1] & [0, 0.3] & [0.9, 1] \\
		$\eevidence$ & [0, 1] & [0.7, 1] & [0, 0.1] \\
		$T_1$        & [0, 1] & [0, 0.33] & [0, 1] \\
		$T_2$        & [0, 1] & [0.7, 1] & [0, 0.1] \\
		$T_3$        & [0, 1] & [\textbf{0.7}, 1] & [0, 0.1] \\
		$E_1$        & [0, 1] & [0, 0.3] & [\textbf{0.9}, 1] 
	\end{tabular}
	\caption{Beliefs under additional assumptions  for Example \ref{example_blaf} (rounded to two digits). Directly constrained beliefs are highlighted in bold.}
	\label{tab:example_blaf}
\end{table}
We could now start adding assumptions and looking at the consequences. For example, let us assume
that the statement of the defendant's girlfriend was very convincing. We could incorporate this by 
adding the constraint $\pi(T_3) \geq 0.7$. The consequences are shown in the third column ($\belief_2$) of Table \ref{tab:example_blaf}. However, if the person on the camera bears strong resemblance to the 
defendant, we may find that the upper belief in $E_1$ is too low. This means that our assumption
is too strong and needs to be revised. Let us just delete the constraint $\pi(T_3) \geq 0.7$ and instead impose a
constraint on $E_1$. Let us assume that there is hardly any doubt that the camera shows the 
defendant. We could incorporate this by adding the constraint $\pi(E_1) \geq 0.9$.
The consequences are shown in the fourth column ($\belief_3$) of Table \ref{tab:example_blaf}.
\end{example}
The choice of probabilities (degrees of belief), weights (relevance) and additional attack or support relations is, of course, subjective. However, arguably, every court decision is subjective in that
the decision maker(s) have to weigh the plausibility and the relevance of the evidence in one way
or another. By making these assumptions explicit in a formal framework, the decision process can become more transparent. Furthermore, by computing probabilities while adding assumptions, possible inconsistencies can be detected and resolved early. Since we restrict to 
linear atomic constraints, computing probabilities can be done within a second even when there are thousands of arguments.

Let us note that our framework also allows defining some simple rules that allow deriving
explanations for the verdict automatically.
For example, the belief in $\innocence$ can be explained directly from the beliefs in
$\ievidence$ and $\eevidence$.
If both $\lbelief(\ievidence) \leq 0.5$ and $\lbelief(\eevidence) \leq 0.5$. our system may report
that the defendant is found innocent because of lack of evidence.
If $\lbelief(\eevidence) > 0.5$, it could report that the defendant is found innocent because
the exculpatory evidence is more plausible than the inculpatory evidence (recall from Proposition
\ref{prop_basic_beliefs} that $\ubelief(\ievidence) \leq 1 - \lbelief(\eevidence)$).
Finally, if $\lbelief(\ievidence)$ is sufficiently large, it could report that the defendant
is found guilty because of the inculpatory evidence.
The belief in $\ievidence$ and $\eevidence$ can then be further explained 
based on the belief in supporting hypotheses and pieces of evidence.
The influence of supporting arguments can be measured by their lower belief bounds and their weight.
To illustrate this, consider again Table \ref{tab:example_blaf}.
For $\belief_1$, the system could report that the defendant is innocent because of lack of 
convincing evidence,
while, for $\belief_2$, it can explain that there is convincing exculpatory evidence. If desired, it
can then further report $T_2$ as the direct explanation and, going backwards, $T_3$ as an additional
explanation.
Similarly, for $\belief_3$, the system could report that the defendant is probably not innocent because
of the inculpatory evidence. Again, the system could give further explanations by going backwards in
the graph. We will discuss the idea in more general form in Section 5.

\section{Adding Additional Structure to BLAFs}

BLAFs can capture a wide variety of cases. However, it is often desirable to add additional structure that captures recurring patterns in legal reasoning.
From a usability perspective, this makes the graph more easily comprehensible and allows modeling different cases in a consistent and standardized way.
From an automated reasoning perspective, it allows adding additional general rules that can automatically derive explanations for decisions.

Two natural subsets of inculpatory evidence are direct (\devidence) and circumstantial (\cevidence) inculpatory evidence. While direct evidence provides direct inculpatory evidence, circumstantial evidence involves indirect evidence that requires multiple inferential steps \cite{fenton2013general}. 
For example, a camera that recorded the defendant while committing the crime can be seen as direct
evidence, while a camera that recorded the defendant close to the crime scene like in Example \ref{example_blaf} can be seen as a piece of circumstantial evidence.
Two prominent categories of circumstantial evidence are \emph{motive}
(the defendant had a reason to commit the crime) and \emph{opportunity}
(the defendant had the opportunity to commit the crime).
Figure \ref{fig:blaf_advanced} shows a refined BLAF.
As indicated by the join of their support edges, the beliefs
in pieces of circumstantial evidence are merged and not considered independently.
Only if both a motive and the opportunity 
(and perhaps some additional conditions)
were present, the defendant should be found guilty.
In contrast, pieces of direct evidence are standalone
arguments for the defendant's guilt.

 Two recurring patterns of exculpatory evidence are \emph{alibi} and \emph{ability}. While an alibi indicates that the 
defendant has not been at the crime scene at the time of the crime, \emph{ability} can contain
pieces of evidence that indicate that the defendant could not have committed the crime, for example,
due to lack of physical strength. Figure \ref{fig:blaf_advanced} shows an extended BLAF with six additional meta-hypotheses.
\begin{figure}[tb]
	\centering
		\includegraphics[width=0.43\textwidth]{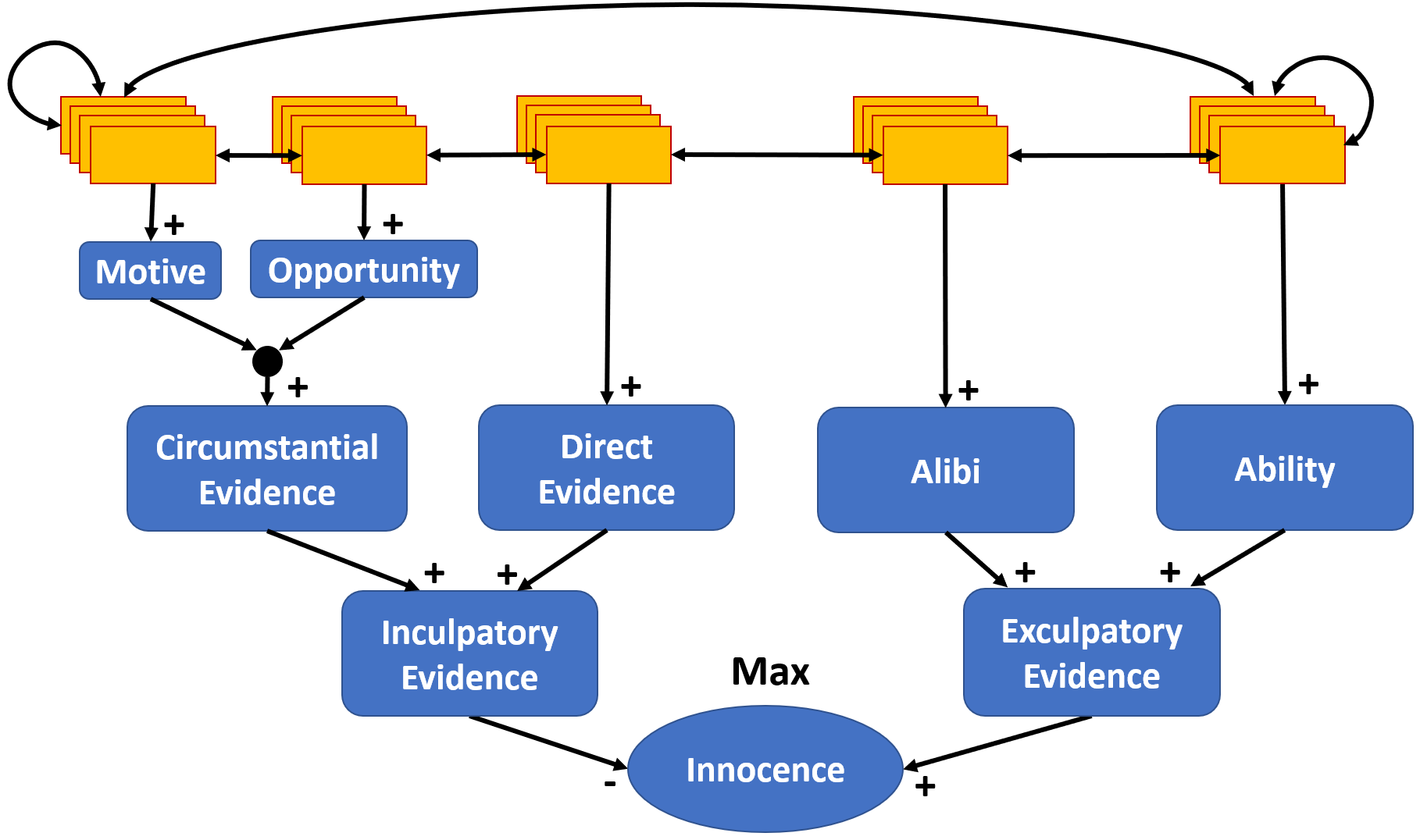}
	\caption{Refined BLAF with additional meta-hypotheses.}
	\label{fig:blaf_advanced}
\end{figure}
As before, we allow edges between all pieces of evidence and subhypotheses, but do not draw all possible direct connections in order to keep the graph comprehensible.
The meaning of the support edges pointing to inculpatory and exculpatory evidence is already defined
by SE in Definition \ref{def_blaf}, item 4. That is the corresponding support relations $(A,B)$
 are associated with the constraint $w((A,B)) \cdot \pi(A) \leq \pi(B)$.
 This constraint could also be naturally used for the evidential edges that point to direct evidence, alibi and ability.
 However, the circumstantial evidence patterns motive and opportunity should not act independently, but complement each other. Neither a motive, nor the opportunity alone, are a good reason to 
 find the defendant guilty. However, if both a good motive and the opportunity are present, this
 may be a good reason. We say that both items together provide \emph{collective support} for the guilt of the defendant. To formalize \emph{collective support}, we can consider a constraint 
 $\weight((\motive, \cevidence)) \cdot \pi(\motive) + 
   \weight((\opportunity, \cevidence)) \cdot \pi(\opportunity)
   \leq \pi(\cevidence)$
such that $\weight((\motive, \cevidence)) + \weight((\opportunity, \cevidence)) \leq 1$.
For example, we could set $\weight((\motive, \cevidence)) = \weight((\opportunity, \cevidence)) = 0.4$.
Then the presence of a strong motive or the opportunity alone cannot decrease the belief in the defendant's
innocent by more than $0.4$ and both together cannot decrease the belief by more than $0.8$.
Opportunity is indeed considered a necessary requirement
for the defendant's guilt in the legal reasoning literature
and motive is, at least, widely accepted as such \cite{fenton2013general}.
Collective support is an interesting pattern in general, so that 
we give a more general definition here.
Given arguments $A_1, \dots, A_n$ (pieces of evidence or sub-hypotheses) that support another argument $B$ 
such that $\sum_{i=1}^n \weight((A_i,B)) \leq 1$, the \emph{collective support constraint} is defined as 
   \begin{description} 
        \item[CS:] $\sum_{i=1}^n \weight((A_i,B)) \cdot \pi(A_i) \leq \pi(B)$.
    \end{description}
The following example illustrates how the additional structure
can be applied.
\begin{example}
\label{example_blaf_extended}
Let us consider a simple robbery case.
The defendant $D$ is accused of having robbed 
the victim $V$. The extended BLAF is shown in Figure \ref{fig:example_blaf_extended}.
\begin{figure}[tb]
	\centering
		\includegraphics[width=0.45\textwidth]{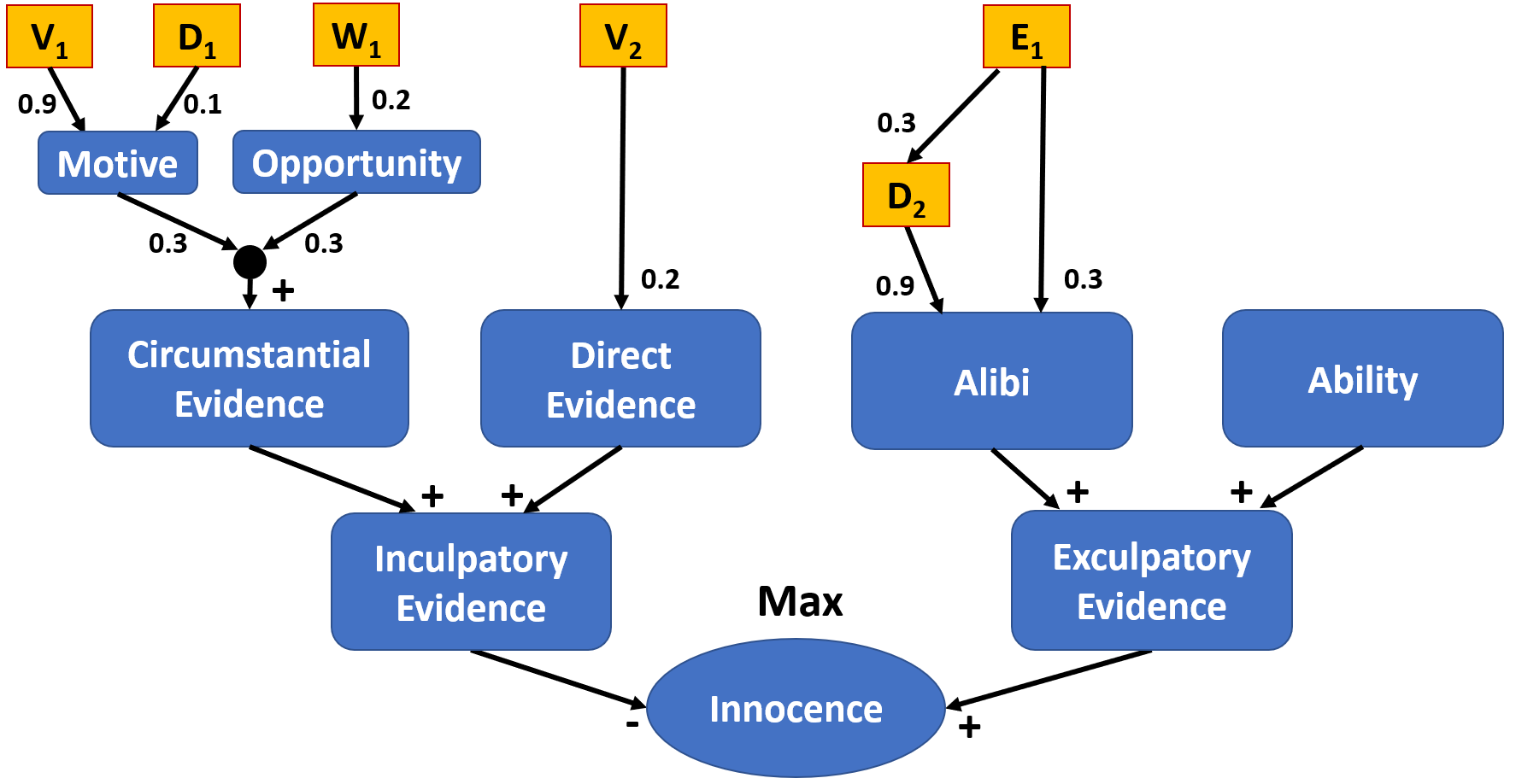}
	\caption{Extended BLAF for Example \ref{example_blaf_extended}.}
	\label{fig:example_blaf_extended}
\end{figure}
Before the crime, $D$ and $V$ met in a bar and had a fight 
about money that $V$ owed $D$.  
$V$ testified that $D$ threatened to get the money one way or another ($V_1$). 
$D$ acknowledged the fight, but denied the threat ($D_1$).
While D's testimony still contains a motive for the crime,
it is now significantly weaker. This can be reflected in the weights.
We could consider a more fine-grained view distinguishing the fight
and the threat and add an attack between the contradicting statements,
but in order to keep things simple, we refrain from doing so.
$V$ testified that he got robbed at 23:30 by a masked person
and that he recognized the defendant based on his voice and 
stature ($V_2$). This can be seen as direct evidence for the crime, but since the accused is of average stature, it should have only a small weight. 
A waiter working at the bar testified that the defendant left the
bar at about 23:00 ($W_1$). This may have allowed the defendant 
hypothetically to commit the crime, but he could have went anywhere, so the weight
should be again low.
The defendant testified that he went to the movie theater 
and watched a movie that started at 23:15 ($D_2$).
If true, this is a strong alibi and should therefore have a large weight.
An employee at the movie theater testified that the defendant
is a frequent guest and that he recalled him buying a drink ($E_1$).
However, he did not recall the exact time. So the alibi is somewhat weak
and should not have too much weight.
We weigh $\motive$ and $\opportunity$ equally  with
$\weight((\motive, \innocence)) = \weight((\opportunity, \innocence)) = 0.3$. 
The influence of the belief in motive and opportunity
on circumstantial evidence is defined by the collective support constraint that we described above. 
All evidential edges $(E,A)$ that originate from a piece of evidence $E$ are associated with the constraint
$w((E,A)) \cdot \pi(E) \leq \pi(A)$.
Figure \ref{fig:example_blaf_extended} shows
the final graph structure and edge weights.

\def\arraystretch{1.2}
\begin{table}

	\centering
	\resizebox{\linewidth}{!}{
	\begin{tabular} {l>{\raggedleft}p{0.09\linewidth}>{\raggedleft}p{0.175\linewidth}>{\raggedleft}p{0.175\linewidth}>{\raggedleft\arraybackslash}p{0.175\linewidth}}
		$\args$ & Basic & $W1$, $E1$ & $W1$, $E1$, $D1$ & $W1$, $E1$, $D1$, $V2$  \\  
		\hline
		$\innocence$    & [0, 1]    & 0.94      & 0.91 & 0.8 \\
		$\ievidence$    & [0, 1]    & [0.06, 0.7]  & [0.09, 0.7] & [0.2, 7] \\
		$\eevidence$    & [0, 1]    & [0.3, 0.94]  & [0.3, 0.91] & [0.3, 0.8] \\
		$\cevidence$    & [0, 1]    & [0.06, 0.7]  & [0.09, 0.7] & [0.09, 0.7]\\
	    $\devidence$    & [0, 1]    & [0, 0.7]  & [0, 0.7] & [0.2, 0.7]\\
        $\alibi$        & [0, 1]    & [0.3, 0.94] & [0.3, 0.91] &[0.3, 0.8]\\
        $\ability$      & [0, 1]    & [0, 0.94] & [0, 0.91] & [0, 0.8]\\
		$\motive$       & [0, 1]    & [0, 1]    & [0.1, 1] & [0.1, 1]\\
		$\opportunity$  & [0, 1]    & [0.2, 1]  & [0.2, 1] & [0.2, 1]\\
		$V1$            & [0, 1]    & [0, 1]    & [0, 1] & [0, 1]\\
		$V2$            & [0, 1]    & [0, 1]    & [0, 1] & \textbf{1}\\
		$D1$            & [0, 1]    & [0, 1]    & \textbf{1} & \textbf{1} \\
		$D2$            & [0, 1]    & [0.3, 1]  & [0.3, 1] & [0.3, 0.89] \\
		$W1$            & [0, 1]    & \textbf{1} & \textbf{1}& \textbf{1}\\
		$E1$            & [0, 1]    & \textbf{1}& \textbf{1}  &\textbf{1}
	\end{tabular}}
	\caption{Belief in $\innocence$ and entailment results
	under additional assumptions for Example \ref{example_blaf_extended} (rounded to two digits). Directly constrained beliefs are highlighted in bold.}
	\label{tab:example_blaf_extended}
\end{table}
Having defined the structure of the graph and the meaning of the edges, we can start to assign beliefs to pieces of evidence. 
Again, without making any assumptions about the beliefs, we can only infer that the degree of belief in $\innocence$  is 1. This is shown in the second column of Table \ref{tab:example_blaf_extended}. 
To begin with, we assume that the testimonies given by the cinema employee and the waiter of the bar are true ($\pi(E1) = 1, \pi(W1) = 1$). The  third column of Table \ref{tab:example_blaf_extended} shows the consequences of
these assumptions. 
We can see, for example, that the alibi $E1$ provides a lower bound for the belief in the exculpatory evidence and thus an upper bound for the beliefs in the inculpatory evidence and the related hypotheses. 
It seems also safe to assume that the defendant did not lie about his participation in the fight, so we the constraint $\pi(D1) = 1$ next. The fourth column in Table \ref{tab:example_blaf_extended} shows the resulting belief intervals. The new support for motive adds to the support of the circumstantial evidence and the lower bound on the belief in the inculpatory evidence is raised. This lowers the belief in the innocence of the accused slightly. Note again that it also decreases the upper bound on the belief in exculpatory evidence indirectly.
Finally, let us assume that the defendant does not lie about having 
recognized the defendant ($\pi(V2) = 1$) (recall that the uncertainty about
the recognition reliability is incorporated in the edge weight).
The fifth column in Table \ref{tab:example_blaf_extended} shows the new
beliefs. We can see that the belief in the defendant's innocence decreases
significantly. If we notice that a larger or smaller change is more plausible,
we could take account of this by adapting the edge weight. 
In this way, legal cases can be analyzed in a systematic way and the
plausibility of assumptions can be checked on the fly by looking at their
ramifications.
\end{example}

In addition to the previously introduced additional categories of meta-hypotheses, another recurring pattern in legal cases are mutually
dependent pieces of evidence. One way to model this in our framework,
is to define a meta-argument that is influenced by the dependent pieces of evidence. The collective support constraint CS is well suited to capture this relationship accurately. We illustrate this with an example from \cite[pp.82-84]{fenton2013general}.
\begin{example}
Let us assume that a person was recorded by two video cameras from
different perspectives at a crime scene. If the person is the defendant,
the defendant should resemble the person on both images.
In the BLAF, we can incorporate the two camera observations as pieces of evidence $\cameraf, \cameras$ supporting a meta-hypothesis $\camera$
that says that the defendant was at the crime scene because of camera evidence. Note that if we use the SE constraint for the evidential edges
from $\cameraf, \cameras$, each of the two cameras would independently determine a lower bound for $\camera$ which seems to strong in this example. 
Instead, we can use the CS constraint that we already used to capture the relationship between opportunity and motive. 
In this example, the CS constraint becomes $\weight((\cameraf, \camera)) \cdot \pi(\cameraf) + \weight((\cameras, \camera)) \cdot \pi(\cameras)\leq \pi(\camera)$,
where $\weight((\cameraf, \camera)) + \weight((\cameras, \camera)) \leq 1$.
For example both camera weights could be set to $\weight((\cameraf, \camera)) = \weight((\cameras, \camera)) = 0.5$ to give equal relevance to both.
Then, if the person resembles the defendant only from one perspective, say we have $\pi(\cameraf)=1$ and $\pi(\cameras)=0$, the induced lower bound on the belief in $\camera$ will be only $0.5$. Only if the belief in
both cameras is larger than $0.5$, the lower bound can be larger than $0.5$. For example, if we have $\pi(\cameraf)=0.7$ and $\pi(\cameras)=0.9$, the induced lower
bound is $0.8$.
\end{example}

\section{Automated Explanation Generation}
\label{sec:explanation}

As we already illustrated at the end of Section 3,
the structure of our framework allows generating
explanations for decisions automatically. 
In general, explaining the meta-hypotheses $\innocence, \ievidence$ and $\eevidence$ is 
easier than explaining the beliefs in other arguments because of their restricted
form.

Note first that the only direct neighbors of $\innocence$ are 
$\ievidence$ and $\eevidence$ and we know that
$\ievidence$ is an attacker and $\eevidence$ is a supporter.
Therefore, we can basically distinguish three cases 
that we already described at the end of Section 3.
\begin{enumerate}
\item $\lbelief(\ievidence) \leq T$ and $\lbelief(\eevidence) \leq 0.5$: The defendant is found
innocent due to lack of evidence.
\item $\lbelief(\eevidence) > 0.5$: the defendant is found innocent because
the exculpatory evidence is more plausible than the inculpatory evidence.
\item $\lbelief(\ievidence) > T$: the defendant
is found guilty because of the inculpatory evidence.
\end{enumerate}
Here, $T$ is a threshold that should usually be chosen from the open interval $(0.5, 1)$. $0.5$ is sometimes 
regarded as the acceptance threshold, but in a legal setting,
it may be more appropriate to choose a larger threshold
like $T=0.75$.

After having received a high-level explanation
of the verdict, the user may be interested in more
details and ask for reasons that explain the
plausibility of inculpatory or exculpatory
evidence.
Explaining $\ievidence$ and $\eevidence$ is more complicated already because
we have an unknown number of neighbors in the graph
now. However, the only neighbors can be supporters
(parents) and $\innocence$ (child).
By Definition \ref{def_blaf}, item 4,
their meaning is encoded by the \emph{SE}-constraint.
Assuming that the user did not add additional
constraints about the relationships between 
$\ievidence$, $\eevidence$ and $\innocence$, we 
can again define some simple rules.
If additional constraints on $\ievidence$ and $\eevidence$ are desirable, these rules may need 
to be refined, of course. 
Otherwise, we can distinguish two cases.
If the user asks for an explanation for the lower belief, we can reason as follows:
a non-trivial lower bound ($>0$) can only result from a supporter with non-trivial lower bound. 
So in this case, we can go through the supporters,
collect those supporters that induce the maximum 
lower bound and report them as an explanation.

The user may also ask for an explanation for the upper belief. A non-trivial upper bound ($<1$) can only result
from a non-trivial bound on the belief in $\innocence$.
Let us assume that we want to explain a non-trivial
upper bound on $\ievidence$. From 
the \emph{IE}-constraint in Definition \ref{def_blaf}, item 4, we can see that this must be caused by
a non-trivial lower bound on $\innocence$.
This lower bound, in turn, must be caused by a 
non-trivial lower bound on $\eevidence$ by our 
assumptions.
We could now report the lower bound on $\eevidence$
as an explanation. A more meaningful explanation
would be obtained by also explaining the lower bound on
$\eevidence$. This can be done as explained before
by looking at the supporters of $\eevidence$.
A non-trivial upper bound on $\eevidence$ can be 
explained in a symmetrical manner.

Generating automatic explanations for the remaining 
sub-hypotheses and pieces of evidence is most 
challenging, but can be done as long as we can make
assumptions about the constraints that are involved.
For example, often the \emph{SE}-constraint gives
a natural meaning to support edges and the weighted \emph{Coherence} constraint gives a natural meaning
to attack edges. Intuitively, they cause a lower/upper
bound on the belief in an argument based on their own
lower belief.
If these are the only constraints that are employed,
explanations for lower bounds can again be generated by
collecting the supporters that induce the largest lower bound.
For explaining the upper bound, we now have to consider
two factors. The first factor are attackers with a non-trivial lower bound. The second factor are
other arguments that are supported and have a non-trivial upper bound
(then a too large belief in the supporting argument would cause an inconsistency).
Therefore, we do not only collect the attacking
arguments that induce the largest lower bound, but 
we also collect supported arguments. We can order the
supported arguments by their upper belief multiplied 
by the weight of the support edge.
If the smallest upper bound from the supported arguments
is $U$
and the largest lower bound from the attacking arguments is $L$,
we report the collected supported arguments as an explanation if $1-U > L$, the collected attacking
arguments as an explanation if $1-U < L$ or
both if it happens that $1-U = L$.

For additional constraints, we may have to refine
these rules again. One important constraint that
we discussed is the \emph{CS}-constraint.
In this case, we have to to treat the supporters involved in this constraint differently since
they all contribute to the induced lower bound.
When collecting supporters for explaining lower bounds
(the supporters are parents),
supporting edges that belong to one \emph{CS}-constraint
have to be considered jointly and not independently.
If they induce a lower bound that is larger than all 
lower bounds caused by an \emph{SE}-constraint,
they can be reported collectively as an explanation.
When collecting supporters for explaining upper bounds
(the supporters are children), the reasoning becomes
more complicated because there can be various interactions
between the beliefs in the involved arguments.
We leave an analysis of this case and more general cases for future work.

\section{Related Work}
Our legal reasoning framework allows explicit formalization of uncertainty in legal decision making. 
Other knowledge representation and reasoning formalisms have been applied for 
this purpose. 
Studies of different game-theoretical tools can be found in \cite{prakken1996dialectical,riveret2007success,roth2007strategic}.
\cite{dung2010towards} proposed a probabilistic argumentation framework where the beliefs
of different jurors are represented by individual probability spaces.
Intuitively, the jurors weigh the evidence and decisions can be made based on criteria like
majority voting or belief thresholds.
One particularly popular approach for probabilistic legal reasoning
are Bayesian networks. \cite{fenton2013general} provide a set of idioms used for the construction of Bayesian networks based on legal argument patterns and apply and discuss their framework for a specific case in \cite{fenton2019analyzing}.  \cite{timmer2017two} developed an algorithm to extract argumentative information from a Bayesian network with an intermediate structure, a support graph and analyze their approach in a legal case study. 
 \cite{vlek2016method} propose a method to model different scenarios about crimes with Bayesian networks using scenario scheme idioms and to extract information about the scenario and the quality of the scenario. 
 
 Determining the weights and beliefs for the edges and items of evidence poses a problem for our framework as well as for other symbolic approaches. For some items of evidence the weights as well as the probabilities can be elicited based on statistical analysis and forensic evidence \cite{kwan2011sensitivity,fenton2012risk,zhang2016expert}. 
To test the robustness of Bayesian networks with respect to minor changes in subjective beliefs, \cite{fenton2013general} propose to apply sensitivity analysis on the nodes in question. In our framework, the impact of subjective beliefs can be analysed in a similar manner, by altering the beliefs which are associated with the evidence or the weights associated with the edges. The automated explanation generation outlined in Section \ref{sec:explanation} can then provide information about the influence that differing beliefs have on hypotheses and sub-hypotheses in the framework. With this the perspective of different agents can be modeled, for example the defense and prosecution perspectives.
The clear structure of argumentation frameworks is well suited for generating explanations
automatically and related explanation ideas have been considered recently in \cite{cocarascu2019extracting,vcyras2019argumentation,zeng2018context}, for example.

In Bayesian networks, inconsistency is usually not an issue
because of the way how they are defined. In contrast, 
in our framework, inconsistencies can easily occur.
For example, if a forensic expert judges both the accuracy
of an alibi and the relevance of a direct piece of evidence with $1$,
our constraints become inconsistent. While this may be inconvenient,
this inconsistency is arguably desirable. This is because the modeling
assumptions are inconsistent and this should be recognized and reported
by the system.
If automated merging of the inconsistent beliefs is desirable,
this can be achieved by different tools. One possibility is to
apply inconsistency measures for probabilistic logics in order to evaluate the
severity of conflicts \cite{de2015measuring,potyka2014linear,thimm2013inconsistency}. 
In order to determine the sources of the inconsistency and their impact,
Shapley values can be applied \cite{hunter2010measure}.
Alternatively, we could replace our exact probabilistic reasoning algorithms
with inconsistency-tolerant reasoning approaches that resolve inconsistencies
by minimizing conflicts \cite{adamcik2014collective,muino2011measuring,potyka2015probabilistic} or based on priorities \cite{potyka2015reasoning}. This would be more convenient for
the knowledge engineer, but the resulting meaning of the probabilities
becomes less clear.

\section{Conclusions and Future Work}

We proposed a probabilistic abstract argumentation framework for automated reasoning in law
based on probabilistic epistemic argumentation \cite{HunterT16,HunterPT2018Arxiv}.
Our framework is best suited for merging beliefs in pieces of evidence and sub-hypotheses.
Computing an initial degree of belief for particular pieces of evidence based on
forensic evidence can often be better accomplished by applying Bayesian networks 
or a conventional statistical analysis. Our framework can then be applied on top
in order to merge the different beliefs in pieces of evidence and subhypotheses
in a transparent and explainable way.
In particular, point probabilities are not required, but imprecise probabilities
in the form of belief intervals are supported as well. 

It is also interesting
to note that the worst-case runtime of our framework is polynomial \cite{potyka2019fragment}.
Bayesian networks also have polynomial runtime guarantees in some special cases,
for example, when the Bayesian network structure is a polytree (i.e., it does not contain cycles when ignoring the direction of the edges).
The polynomial runtime  in probabilistic epistemic argumentation is guaranteed by
restricting to a fragment of the full language. This fragment is sufficient for
many cases and is all that we used in this work. However, sometimes it may be
necessary to extend the language. For example, instead of talking only about the 
probabilities of single pieces of evidence and subhypotheses, we may want to talk
about the probabilities of logical combinations. Similarly, one may want to merge
beliefs not only in a linear, but in a non-linear way. Both extensions are difficult
to deal with, in general. However, it seems worthwhile to study such cases in more
detail in order to identify some other tractable special cases.

Another interesting aspect for future work is extending the automated support tools for
designing and querying our legal argumentation frameworks. As explained in Section 5,
the basic framework can be explained well automatically. However, when beliefs are
merged in more complicated ways like by the collective support constraint, a deeper
analysis is required. We will study explanation generation for collective support 
and other interesting merging patterns in more detail in future work. 
For the design of the framework, it may also be helpful to generate explanations
for the sources of inconsistency. As explained in the related work section,
a combination of inconsistency measures for probabilistic logics and Shapley values
seems like a promising approach that we will study. It is also interesting
to apply different approaches for inconsistency-tolerant reasoning in order to
avoid inconsistencies altogether. However, while these approaches usually can
give some meaningful analytical guarantees, it is important to study empirically
if these guarantees 
are sufficient in order to guarantee meaningful results in legal or other applications.

\bibliographystyle{kr}
\bibliography{references}
\end{document}